\documentclass[]{article}
\usepackage{amsmath}
\usepackage{amsfonts}
\usepackage{amssymb}
\usepackage{amsthm}
\usepackage{mathrsfs}
\usepackage{enumerate}
\usepackage{xspace}
\usepackage[hidelinks]{hyperref} % The hidelinks is to get rid of borders around refernces.
\usepackage{xcolor} % for the \merk command I define below (copying Doerr)

\usepackage[backend=bibtex,style=numeric]{biblatex}% https://tex.stackexchange.com/questions/306218/simple-way-to-use-bibtex

\newtheorem{theorem}{Theorem}
\newtheorem{proposition}[theorem]{Proposition}
\newtheorem{lemma}[theorem]{Lemma}

\newcommand{\norm}[1]{\left\lVert#1\right\rVert}

\newcommand{\oea}{\mbox{$(1 + 1)$~EA}\xspace}

\newcommand{\onemax}{\textsc{OneMax}\xspace}
\newcommand{\LO}{\textsc{Leading\-Ones}\xspace}
\newcommand{\leadingones}{\LO}

\newcommand{\blockleadingones}{\textsc{Block\-LeadingOnes}\xspace}

\newcommand{\jump}{\textsc{Jump}\xspace}

\newcommand{\diststepstones}{\mbox{$\textsc{DistantSteppingStones}_p$}\xspace}

\newcommand{\sspace}{\ensuremath{{\{0,1\}^n\xspace}}}
\newcommand{\ones}{\ensuremath{\mathbf{1}^n}\xspace}

\newcommand{\eps}{\varepsilon}
\newcommand{\R}{\mathbb{R}}
\newcommand{\fp}{\ensuremath{f_{p}}}
\newcommand{\xstar}{\ensuremath{x^*}}
\newcommand{\pso}{\ensuremath{p_{S,\ones}}}

\newcommand{\floor}[1]{\lfloor #1 \rfloor}
\newcommand{\ceil}[1]{\lceil #1 \rceil}

\DeclareMathOperator{\expectation}{E}

\DeclareMathOperator*{\argmax}{arg\,max}
\DeclareMathOperator*{\argmin}{arg\,min}

\title{All Constant Mutation Rates for the \\ $(1+1)$ Evolutionary Algorithm}
\author{Andrew James Kelley}
\begin{document}

\maketitle

\begin{abstract}
	For every mutation rate $p \in (0, 1)$, and for all $\varepsilon > 0$, 
	there is a fitness function $f : \{0,1\}^n \to \mathbb{R}$ with a unique maximum for which the 
	optimal mutation rate for the $(1+1)$ evolutionary algorithm 
	on $f$ is in $(p-\varepsilon, p+\varepsilon)$. In other words, the set of optimal
	mutation rates for the $(1+1)$ EA is dense in the interval $[0, 1]$.
	To show that, this paper introduces $\textsc{DistantSteppingStones}_p$,
	a fitness function which consists of large plateaus separated by large fitness valleys.

	%Let $M$ be the set of all mutation rates in $[0, 1]$ that are optimal 
	%for the $(1+1)$-EA on some optimization problem $f : \{0, 1\}^n \to \mathbb{R}$. 
	%This paper shows that $M$ is dense in the unit interval $[0, 1]$.
	%The $\textsc{DistantSteppingStones}_\alpha$ fitness functions
	%are introduced to prove the result on optimal mutation rates.
\end{abstract}

\section{Introduction}
In the last couple of decades, much progress has been made in analyzing the runtime of evolutionary algorithms (EAs). The theorems
proved about EAs have led to insights that have been translated to improvements on their use in solving real problems.

For some fitness functions, we know the precise expected runtime for different mutation rates. In some cases, 
we can calculate which static mutation
rate minimizes the expected runtime (thus called the optimal mutation rate). Of the set of possible mutation rates, 
it is natural to ask which ones are 
possible as optimal mutation rates. Knowing what kinds of fitness functions have what optimal mutation rates helps us
better understand what mutation rates should be used.

As stated in the abstract of \cite{DoerrLMN17}, the generally recommended mutation rate for the \oea is $1/n$. It is known that
the optimal mutation rate for the \oea on \onemax is $1/n$, and in fact, it was shown in \cite{GiessenW17} that as long as $\lambda$
isn't too large (about logarithmic in $n$ or smaller), then $1/n$ is optimal for the $(1+\lambda)$ EA too.
See \cite{Ochoa02} for further discussion on using a mutation rate of $p = 1/n$.

For a much larger optimal mutation rate, \cite{JansenW00} gives a problem where the optimal mutation rate for the \oea
is $\Theta(\log n /n)$ for some appropriately chosen constant. Here, a mutation rate that is not $\Theta(\log n / n)$ results
in superpolynomial expected runtime. On the other hand, a mutation rate of $c \ln n / n$ with $c = 1/(4 \ln 2) \leq 0.361$
results in an expected runtime of $O(n^{2.361})$. Note that the fitness function defined in \cite{JansenW00} has a jump in it.
%\merk{TODO: Look at reference 11 from the paper ``Optimal Fixed and Adaptive Mutation Rates
%for the LeadingOnes Problem''. See what I highlighted on page 2 of my pdf.} 

The optimal mutation rate for the \oea on \leadingones was calculated in \cite{BottcherDN10} to be approximately $1.5936/n$.
It was shown in \cite{doerrK2023fourier} that the optimal mutation rate on  \blockleadingones (a generalization of \leadingones)
is this same value $c/n$, where $c = \argmin_{x > 0} (e^x-1)/x^2$; this optimal mutation rate is valid only if the block length is not too large.

If the fitness function has a local max, then the optimal mutation rate can be even higher. The paper \cite{DoerrLMN17} 
gives the optimal mutation rate for the $\jump_m$ function as $m/n$, where $m = o(n)$. These mutation rates are still $o(1)$.

Many researchers might consider mutation rates greater than 1/2 to be contrived. However, perhaps a mutation rate 
close to 1/2 is
what should be considered more odd for an EA than a mutation rate close to 1. Indeed, a mutation rate of 1/2 means that no memory of past individuals
is preserved. In that case, the \oea reduces to sampling $\sspace$ uniformly at random. Such an EA can never build off
of progress already made.
On the other hand, a mutation rate of $1 - \eps$ for small $\eps$ means that the current individual carries with it the 
memory of what recent previous individuals were approximately equal to. Indeed, in that case, with high probability, 
all recent ancestors of $x$ were either very similar to $x$ or
were very similar to what you get by flipping all bits of $x$. As a result, the \oea with a mutation rate of $1 - \eps$ can actually
make progress on some fitness functions.

%Although not yet published in a journal, it can be shown that the optimal mutation rate for the \oea on \needle
%is larger than 1/2 for all $n$. If we denote the optimal mutation rate for \needle on $\sspace$ by $p_n$, then
%it can be proved that 
%$p_n > 1/2$ for all $n$, and $p_n$ decreases monotonically, and also $p_n \to 1/2$ as $n \to \infty$. These facts are (nontrivial) 
%consequences of the theorem giving the exact expected runtime of the \oea on \needle for mutation rate $p$, 
%which was shown in \cite{doerrK2023fourier} to be equal to 
%$\displaystyle{\sum_{j=1}^n \binom{n}{j} \frac{1}{1 - (1-2p)^j}}$.

%Hence, there is a standard optimization problem having a mutation rate close 1/2. Consequently, 
%the present author thinks that it is not too theoretical
%to consider that there are optimization problems whose optimal mutation rate is near any $p \in (0,1)$.

The contribution of the present paper is to prove the following:
\begin{theorem}
  For every $p \in (0, 1)$ and for every $\eps > 0$, there exists a fitness function $f_n :\{0, 1\}^n \to \mathbb{R}$
   with a unique maximum such that the optimal mutation rate for the \oea on $f_n$ is in $(p -\eps, p+\eps)$.
\end{theorem}
The way this is shown is by defining \diststepstones, and showing that as $n \to \infty$, its optimal mutation rate
approaches $p$. 
%Note that as is standard, we say that a (static) mutation rate $p$ is the optimal mutation for a fitness function $f_n$ if $p$ is what minimizes 
%the expected time of hitting the optimum.

This paper is organized as follows. Section \ref{sec:notation_and_setup} introduces notation and defines 
\diststepstones.   Section~\ref{sec:main_results} states the main results proved, and Section~\ref{sec:proof_of_results}
proves these results.

\section{Notation and defining DistantSteppingStones}
\label{sec:notation_and_setup}

For a function $f : \R \to \R$, and for $k \geq 1$, define $f^{(k)}(x)$ recursively as $f(f^{(k-1)}(x))$, where $f^{(0)}(x) = x$.

For $x \in \{0,1\}^n$, we let $\norm{x}$ denote the number of 1's in $x$; we have $\norm{x} = \sum_{i=1}^n x_i$.
A function of unitation is a fitness function that depends only on $\norm{x}$ rather than $x$. Let $p \in (0, 1)$. 
We will define \diststepstones, a function of unitation. We specify the successive stepping stones,
which are plateaus of constant fitness separated by large valleys. We specify them in order from highest fitness to lowest. To do that,
define the function $f_p : [0, 1] \to [0, 1]$ by
\[
  \fp(x) = p (1-x) + (1-p) x.
\]

Next, consider the following sequence. Let $s_0 = 1$, and for $k \geq 1$, 
let $s_{k} = \fp^{(k)}(s_0)$. Note that $s_1 = 1-p$ and $s_2 = (1-p)^2 + p^2$.
Lemma~\ref{lem:exp_fast_monotonic_decr} implies 
 that $s_k \to 1/2$ as $k \to \infty$. If $p < 1/2$, then
$s_k$ decreases monotonically, and if $p > 1/2$, then $s_k$ bounces between 
being greater and less than 1/2, but still we have in this case that $|s_k- 1/2|$ decreases
monotonically. 
%This is proved in Lemma~\ref{lem:exp_fast_monotonic_decr}

Fix $n$, a positive integer. Let $N$ be the smallest integer
such that $\floor{ns_N} = n/2$. For $k =0, 1, \ldots, N$, 
we define the $k$th stepping stone of the fitness function \diststepstones
as the set $S_k \subseteq \{0, 1\}^n$ defined
as
\[
S_k = \{x \in \{0, 1\}^n : \norm{x} = \floor{ns_k} \}.
\]
Notice that $S_0 = \{\ones \}$, a singleton of just the all-ones string, which we will
\emph{not} refer to as a ``stepping stone'' (despite the letter ``$S$'').
The biggest stepping stone is $S_N$, the set of all strings having $\lfloor n/2 \rfloor$ ones.
For $x \in \{0, 1\}^n$ define the fitness function as
\[
\diststepstones(x) = 
\begin{cases}
N+1 - k, & \text{if } x \in S_k\\
0, & \text{otherwise}.
\end{cases}
\]
The subscript $p$ reminds us that it comes from $\fp$, which was used to define
the numbers $s_k$ for the sets $S_k$.

Note that $S_N$ is typically the first stepping stone that the EA finds. When using an optimal
mutation rate, it usually goes to $S_{N-1}$ (eventually)  and then $S_{N-2}$ etc. These observations
are not needed though and so will not be proved.

The sets $S_k$ and function $\fp$ came from viewing how the \oea progresses, starting form $\ones$,
assuming a constant mutation rate of $p$ (and a constant fitness function on $\{0,1\}^n$). 
The expected number of bits in the $k$th step is $ns_k = n\fp^{(k)}(1)$. To get to $\ones$, the \oea
tries to trace these steps backwards. %Recall that $p$ is constant. 
Each stepping stone $S_k$ is a local maximum. These are
separated from each other by fitness valleys of constant fitness 0. The distance between stepping
stones is large if $n$ is large compared to $\max\{1/p, 1/(1-p)\}$.

\section{Main results}
\label{sec:main_results}

%Theorem~\ref{thm:opt_mute_rate_is_p} is the main theorem.

\begin{theorem}
\label{thm:opt_mute_rate_is_p}
  The optimal mutation rate for \diststepstones approaches $p$ as $n \to \infty$. 
\end{theorem}
This theorem is a direct consequence of the following theorem:
\begin{theorem}
\label{thm:exponential_runtimes}
  For mutation rate $q \in (0, 1)$, let $T(q)$ be the runtime of the \oea on \diststepstones. 
 Then $\expectation[T(p)] = O(\alpha^n \log n)$, with $\alpha = \frac{1}{p^{p}(1-p)^{(1-p)}}$. Also, for $q \neq p$,
 $\expectation[T(q)] = \Omega(\eta^n)$ for some $\eta > \alpha$.
\end{theorem}

The basic idea for the first part of Theorem~\ref{thm:exponential_runtimes} is that the \oea finds a stepping stone
in at most $O(\alpha^n)$ time, and then it takes $O(\alpha^n)$ time to get to the next stepping stone. There are
logarithmically many stepping stones (except when $p = 1/2$).
The next section is devoted to proving Theorem~\ref{thm:exponential_runtimes}.

\section{Proof of Theorem~\ref{thm:exponential_runtimes}}
\label{sec:proof_of_results}
For this whole section, the value $p \in (0, 1)$ is constant. 
We will need a number of lemmas, but first, we handle the special case $p = 1/2$.

\begin{proposition}
  Consider \diststepstones with $p = 1/2$. Then $\expectation[T(1/2)] = 2^n$.
  Also, for $q \neq 1/2$, $\expectation[T(q)] = \Omega(\eta^n)$, for $\eta = (q(1-q))^{-1/2}$. Since $\eta > 2$, the optimal mutation 
  rate approaches 1/2.
  \end{proposition}
\begin{proof}
Note that $S_1$ is the only stepping stone, 
which is the set of all $x$ with $\norm{x} = n/2$.
Using the mutation rate of 1/2, $T(1/2)$ is a geometric random variable
with expectation exactly $2^n$ (whether or not it finds the one stepping stone first).

Let $S$ be the event that the EA finds the unique stepping stone before \ones. 
By Lemma~\ref{lem:stepping_stone_first}, $\Pr(S) \geq 1/2$. (Actually $\Pr(S)$ is exponentially close to 1, since $|S_1| = \Theta(2^n/\sqrt{n})$, 
but we don't need that.)
Using the law of total expectation,
\[
\begin{aligned}
  \expectation[T(q)] &= \expectation[T(q) | S] \Pr(S) + \expectation[T(q) | \overline{S}] \Pr(\overline{S}) \\
  			      &\geq \expectation[T(q) | S] \Pr(S)\\
			      &\geq \frac{1}{2} \expectation[T(q) | S].
\end{aligned}
\]
But once $x \in S_1$, the probability of reaching the optimum in the next step
equals  $q^{n/2}(1-q)^{n/2}$. Lemma~\ref{lem:sqrt_ineq} tells us that $\sqrt{q(1-q)} < 1/2$, and so $\expectation[T(q) | S] \geq (1/\sqrt{q(1-q)})^n$, and this is exponentially
larger than $2^n$. Therefore, so is $\expectation[T(q)]$. 

 These asymptotics show that for \diststepstones with $p = 1/2$, the optimal mutation rate approaches 1/2 (perhaps equaling 1/2).
\end{proof}

We need the following elementary lemma:
\begin{lemma}
\label{lem:sqrt_ineq}
  For all $x \in (0, 1)$, we have $\sqrt{x(1-x)} \leq 1/2$, and we also have $1/2 \leq x^x(1-x)^{1-x}$. Therefore
  \[
    \sqrt{x(1-x)} \leq x^x(1-x)^{1-x}.
  \]
  The inequality is an equality only for $x = 1/2$.
\end{lemma}
\begin{proof}
  The inequality $\sqrt{x(1-x)} \leq 1/2$ is equivalent to $0 \leq (x - 1/2)^2$, which proves the first inequality, and this is an
  equality only for $x = 1/2$.
  
  To show $1/2 \leq x^x(1-x)^{1-x}$, we need only show 
  \[
    \ln(1/2) \leq x\ln(x) + (1-x)\ln(1-x).
  \]
  This follows from elementary calculus, by showing that the function $g(x) = x\ln(x) + (1-x)\ln(1-x)$ has a minimum at $x = 1/2$;
  this follows because $g'(x) = \ln(x) - \ln(1-x)$. Since $g(1/2) = \ln(1/2)$, we are done with the second inequality.
\end{proof}

The following almost trivial lemma is far weaker than one might want, but this paper has no need of it being any stronger:
\begin{lemma}
\label{lem:stepping_stone_first}
	Let any (constant) mutation rate $q \in (0, 1)$ be used. Let $S$ be the event that
	the EA reaches some stepping stone before \ones. Then $\Pr(S) \geq 1/2$. (Here, it is assumed that $n \geq 2$.)
\end{lemma}
\begin{proof}
  Since $n \geq 2$, there is at least one point $y \in S_N$ (the largest stepping stone). The EA starts with an individual chosen uniformly at random. 
  If the fitness function were constant, then by symmetry, for all $x, z \in \sspace$ with $x \neq z$, the probability of visiting 
  $x$ before $z$ is 1/2. 
  
  On \diststepstones, the EA will either reach a stepping stone before \ones or vice versa. Having more than one point
  in the set of stepping stones cannot decrease the probability of reaching a stepping stone.
  We conclude that $\Pr(S) \geq 1/2$. 
\end{proof}

For the rest of this paper, $p \in (0, 1)$ is constant, with $p \neq 1/2$.
\begin{lemma}
	\label{lem:recursive_fn_simplification}
	Consider $\fp(x) = (1 - p)x + p(1-x) = (1-2p)x + p$. For $k \geq 1$, we define $\fp^{(k)}(x)$ recursively as $\fp(\fp^{(k-1)}(x))$, where $\fp^{(0)}(x) = x$.
	%  Let $h(x) = (1-2p)x$.
	Then for all $k \geq 1$, we have $\fp^{(k)}(x) = (1-2p)^k (x - 1/2) + 1/2$.
\end{lemma}
\begin{proof}
	For $k \geq 1$, let $h^{(k)}(x) = (1-2p)^k(x)$.
	Case $k=1$ for $\fp^{(k)}(x)$ is true since $(1-2p)x + p = (1-2p)(x - 1/2) + 1/2$. By induction, assume the formula for $\fp^{(k)}(x)$
	holds for any particular $k$.
	Then applying the inductive hypothesis (that $\fp^{(k)}(x) = h^{(k)}(x - 1/2) + 1/2$) and using case $k=1$ gives us that
	\[
	\begin{aligned}
	   \fp^{(k+1)}(x) &= \fp(\fp^{(k)}(x)) \\
	   &= h(\fp^{(k)}(x) - 1/2) + 1/2 \\
	   &= h(h^{(k)}(x - 1/2) + 1/2 - 1/2) + 1/2 \\
	   &= h^{(k+1)}(x-1/2) + 1/2,
	\end{aligned}
	\]
	which proves the equation $\fp^{(k)}(x) = (1-2p)^k (x - 1/2) + 1/2$ by definition of $h^{(k)}(x)$.
\end{proof}

%We will need a slight strengthening of Lemma~\ref{lem:recursive_fn_simplification}: NOTE: but this isn't just a generalization of that lemma
\begin{lemma}
\label{lem:calling_fns_f_repeatedly_const_times}
  Let $\rho \in (0, 1)$. For $j \geq 1$, let $\rho_j \in (-1,1)$ with $|\rho_j| \leq \rho $. For $k \geq 1$, define 
  $f^{(k)}(y) = (y- 1/2) \left(\prod_{j=1}^k \rho_j\right) + 1/2$.
  Let $\eps > 0$, and let $\beta \in [0, 1]$. Then there exists a constant $M(\rho, \eps)$ such that for $t \geq M(\rho, \eps)$,
  \[
    |f^{(t)}(\beta) - 1/2| < \eps.
  \]
\end{lemma}
\begin{proof}
  We have $|f^{(t)}(\beta) - 1/2| =  |(\beta - 1/2)\prod_{j=1}^t\rho_j | \leq \rho^t |\beta - 1/2| \leq \rho^t$. Then $\rho^t \leq \eps$ if and only if
  $t \geq \frac{\log \eps}{\log \rho}$. We may take $M(\rho, \eps) = \frac{\log \eps}{\log \rho}$.
\end{proof}

\begin{lemma}
	\label{lem:exp_fast_monotonic_decr}
	%Let $s_0 = 1$, and for $n \geq 0$,  let $s_{n+1} = \fp(s_n)$. 
	Let $\fp$ be as before.
	Let $s_0 = 1$, and for $k \geq 1$, let $s_{k} = \fp^{(k)}(s_0)$.
	Then $s_k \to 1/2$ exponentially fast, and 
	$|s_k - 1/2|$ decreases monotonically.
\end{lemma}
\begin{proof}
	This follows from the equation $\fp^{(k)}(x) = (1-2p)^k (x - 1/2) + 1/2$ in Lemma~\ref{lem:recursive_fn_simplification}.
\end{proof}

\begin{lemma}
	\label{lem:N_is_Theta_log_n}
	Assume (still) that $p \neq 1/2$. Let $N$ be the minimum integer such that $\floor{ns_N} = n/2$. Then $N = \Theta(\log(n))$.
\end{lemma}
\begin{proof}
	By Lemma~\ref{lem:exp_fast_monotonic_decr}, the sequence $s_n$ approaches 1/2 exponentially fast.
	The present lemma follows as a consequence.
\end{proof}

\begin{lemma}
	\label{lem:markov_chain_lemma_stochastic_domination}
	Let $\{X_n\}_{n=1}^\infty$ be a Markov chain on set $\Omega$. Let $\xstar \in \Omega$. Let $A \subseteq \Omega$, and suppose that if 
	$X_{n-1} \in A$ and $X_n \notin A$, then $X_n = \xstar$. 
	Suppose there exists $p > 0$ such that
	$\Pr[X_n = \xstar | X_{n-1} \in A ] \leq p$ for all $n$. Then $\expectation[T | X_1 \in A]$, the expected first hitting time 
	of \xstar given $X_1 \in A$, is bounded by
	\[
	   \expectation[T | X_1 \in A] \geq \frac{1}{p}.
	\]
\end{lemma}
\begin{proof}
  Let $T_g$ be a geometric random variable with parameter $p$. Then
  $T$ given $X_1 \in A$ stochastically dominates  $T_g$. (A definition of stochastic domination can be found in Definition 1.8.1 in
  \cite{DoerrN20}.) Therefore
  \[
    \expectation[T|X_1 \in A] \geq \expectation[T_g] = \frac{1}{p}.
  \]
\end{proof}

\begin{lemma}
\label{lem:p_greater_than_half_what_is_pso}
	Assume $p > 1/2$.
	Let any mutation rate $q$ be used with $q \neq p$. Assume the EA is at a stepping
	stone. Let $\pso$ denote the probability that the next offspring is the optimum. 
	If $q \geq 1/2$, then 
	\begin{equation}
	\label{eq:inequality_with_large_q}
	\pso \leq q^{\ceil{p n}} (1-q)^{\floor{(1-p) n}}.
	\end{equation}
	If $q < 1/2$, then with $\beta := (1-p)^2 + p^2$, we have
	\begin{equation}
	\label{eq:inequality_with_small_q}
	\pso \leq q^{\ceil{(1-\beta)n}} (1-q)^{\floor{\beta n}}.
	\end{equation}
\end{lemma}
\begin{proof}
	Recall that $s_0 = 1$, and for $n \geq 0$,  $s_{n+1} = \fp(s_n)$. Recall also that as $k$ changes, 
	the $k$th stepping stone $S_k$
	bounces back and forth from having more or fewer 1s than $n/2$ (because $p > 1/2$).
	
	First assume $q \geq 1/2$. If $q = 1/2$, this result is trivial, and so assume $q > 1/2$. 
	Then the EA has the largest probability to reach the optimum when the number of bits
	to not flip is minimized. If the current individual has $k$ 1s, then the probability to reach the optimum in
	the next step is
	\begin{equation}
	\label{expression:exact_prob_to_opt_in_one_step}
	q^{n-k}(1-q)^k,
	\end{equation}
	which, because $q > 1/2$, is maximized when $k$ is minimal. Recall that $s_1 = f_p(1) = 1 - p$. 
	The stepping stone with the minimal 
	number of 1s has $\floor{n s_1} = \floor{(1-p)n}$ 1s. This proves Equation~(\ref{eq:inequality_with_large_q}).
	
	Next, assume $q < 1/2$. The quantity from expression (\ref{expression:exact_prob_to_opt_in_one_step}) is still
	the probability of reaching the optimum in the next step, but now, since $q < 1/2$, it is maximized when $k$
	is maximal. The stepping stone with the maximal number of 1s has $\floor{n s_2}$ 1s in it, but 
	$s_2 = f_p(s_1) = (1-p)^2 + p^2 = \beta$. This proves Equation~(\ref{eq:inequality_with_small_q}).	
\end{proof}

\begin{lemma}
	\label{lem:maximize_simple_fns}
	Let $\alpha \in (0, 1)$. Let $g_1(x) = (1-x)^{1-\alpha}x^{\alpha}$ and $g_2(x)  = x^{1-\alpha}(1-x)^{\alpha}$. 
	Then 
	\[
	\argmax_{x \in [0, 1] }g_1(x) = \alpha, \quad \quad \text{and} \quad \quad \argmax_{x \in [0, 1] }g_2(x) = 1 - \alpha,
	\]
	and these arguments are the only ones in $[0,1]$ that maximize the functions.
\end{lemma}
\begin{proof}[Proof sketch]
	This follows from basic calculus by taking the derivative. By symmetry, the statement for $g_1$
	implies the statement for $g_2$.	
\end{proof}

\begin{lemma}
\label{lem:ineqs_with_beta}
	Let $p, q \in (0, 1)$. Let $\beta := (1-p)^2 + p^2$.
	Then 
	\[
	  q^{1-\beta}(1-q)^\beta < (1-p)^{1-p} p^p.
	\]
\end{lemma}
\begin{proof}
	By Lemma~\ref{lem:maximize_simple_fns}, $q^{1-\beta}(1-q)^\beta < (1-\beta)^{1-\beta}\beta^\beta$, and
	so it is sufficient to prove that
	\begin{equation}
	\label{ineq:beta_to_p}
	   (1-\beta)^{1-\beta}\beta^\beta < (1-p)^{1-p} p^p.
	\end{equation}
	For $x \in [0, 1]$, let $h(x) = (1-x)^{1-x}x^x$. First, note that $h(x)$ is symmetric about the line $x = 1/2$.
	Next, a little bit of calculus shows that $h(x)$ is decreasing on $(0, 1/2)$
	and so increasing on $(1/2, 1)$. We know that $\beta = s_2$, and so by Lemma~\ref{lem:exp_fast_monotonic_decr},
	it is closer to 1/2 than how close $s_1  = 1 - p$ is. Therefore, $h(\beta) < h(1-p)$, which
	simplifies to (\ref{ineq:beta_to_p}).
\end{proof}

\begin{lemma}
\label{lem:p_less_than_half_what_is_pso}
  Let $p < 1/2$. Let $\pso$ denote the probability of reaching \ones\ in the next step, given that the current individual $x$ is one of the
  stepping stones (so with $x \neq \ones$). Using a mutation rate $q \leq 1/2$, we have
  \[
     \pso \leq q^{\lceil pn \rceil}(1-q)^{\lfloor (1-p)n \rfloor}.
  \]
  Using a mutation rate $q > 1/2$, we have with $\gamma = \sqrt{q(1-q)}$,
  \[
     \pso \leq \gamma^n.
  \]
\end{lemma}
\begin{proof}
  This result is trivial for $q = 1/2$.  Assume $q < 1/2$. Then $\pso$ is largest when there are the fewest number of 0s to flip. This happens in stepping stone $S_1$,
  for which by definition, each $x \in S_1$ has $\norm{x} = \lfloor (1-p)n \rfloor$. Therefore $\pso$ is bounded by the transition probability from $S_1$
  to \ones. This latter probability equals $q^{\lceil pn \rceil}(1-q)^{\lfloor (1-p)n \rfloor}$, which concludes this case.
  
  Next, assume $q > 1/2$. Now, the transition probability from a stepping stone to \ones\ is largest for $S_j$ when $x \in S_j$ has the most zeros.
  This happens for $S_N$, where $x \in S_N$ has $\norm{x} = n/2$. Therefore, \pso\ is bounded by the transition probability from $S_N$ to \ones.
  This latter probability equals $q^{n/2}(1-q)^{n/2}$, which simplifies to $\gamma^n$ for $\gamma = \sqrt{q(1-q)}$.
\end{proof}

\begin{lemma}
\label{lem:squishing_fn_f_using_rho}
  Let $x \in \sspace$, and assume $\norm{x} = n/2 \pm \alpha n$ for some $\alpha \in (0, 1/2]$. Let $\eps \in (0, \alpha]$.
  Let $x'$ be the individual produced by using mutation rate $p$. Choose $\delta > 0$ such that $\rho := |1-2p| + \delta/\eps \in (0, 1)$.
  Then with probability $c = 1 - e^{-dn}$ for some constant $d > 0$,
  there exists a function $f(y) = \rho_1(y - 1/2) + 1/2$ with $\rho_1 \in (-1, 1)$ and $|\rho_1| \leq \rho$ such that
  \[
     \norm{x'} = f(1/2 + \alpha) n.
  \]
\end{lemma}
\begin{proof}
  Let $\norm{x} = n/2 + \alpha n$. (The case for $-\alpha$ is similar.)
  Let $X$ be the number of 1s flipped in $x$, and let $Y$ be the number of 0s flipped. Then $X$ and $Y$ are binomial
  random variables with $n/2 + \alpha n$ and $n/2 - \alpha n$ bernoulli trials respectively and probability $p$.
  %Choose $\delta > 0$ such that $|1-2p| + \delta/\eps \in (0, 1)$.
  Let $c$ denote the probability that 
  \[
    |X - (n/2 + \alpha n)p| < \frac{\delta}{2} n \text{\quad and \quad} |Y - (n/2 - \alpha n)p | < \frac{\delta}{2}  n.
  \]
  Due to the concentration of binomial random variables near their mean, $c$ is exponentially close to 1; indeed,
  by inequalities 1.10.5 and 1.10.12 of \cite{DoerrN20}, there exist $d > 0$ such that $c = 1-e^{-dn}$ %%%
  and with probability $c$,
  \[
     X = \expectation(X) + \delta_1 n, \text{\quad and \quad} Y = \expectation(Y) + \delta_2n,
  \]
  for some $\delta_1$ and $\delta_2$ with $|\delta_1|, |\delta_2| < \delta/2$. % Let $\delta = - \delta_1 + \delta_2$. Then
 
  We have $\norm{x'} = \norm{x} - X + Y$, and with probability $c$,
  \[
    \begin{aligned}
     \norm{x'} &= \frac{n}{2} + \alpha n - \expectation[X] - \delta_1 n + \expectation[Y] + \delta_2 n \\ 
      &= \frac{n}{2} + \alpha n - \left(\frac{n}{2} + \alpha n\right)p - \delta_1 n+ \left(\frac{n}{2} - \alpha n\right)p + \delta_2 n \\
      &= \frac{n}{2} + \alpha n - 2 \alpha n p - \delta_1 n + \delta_2 n\\
      &= \left((1-2p)\alpha + \frac{1}{2} - \delta_1 + \delta_2 \right) n,
    \end{aligned}
  \] 
  and this equals %$((1-2p)(1/2+\alpha) + p - \delta_1 + \delta_2)n$, which is just 
  $f(1/2 + \alpha)n$, where 
  $f(y) = \rho_1 (y - 1/2) + 1/2$ and
  $\rho_1 := 1-2p + (\delta_2 - \delta_1)/\alpha $. We need only show that $|\rho_1| \leq \rho$. 
  
  We have
  \[
     \begin{aligned}
     |\rho_1| &\leq |1 - 2p| + \frac{|\delta_1| + |\delta_2|}{\alpha} \\
     &\leq |1 - 2p| + \frac{|\delta_1| + |\delta_2|}{\eps} \\
     &\leq |1 - 2p| + \frac{\delta/2 + \delta/2}{\eps} \\
     &=  |1 - 2p| + \frac{\delta}{\eps} = \rho,
     \end{aligned}
  \]
  which is what we needed to show. 
\end{proof}

%\merk{TODO: I should try to make the proof of the following a little easier to follow.}
\begin{lemma}
\label{lem:in_constant_steps_in_the_middle}
   Let $x \in \sspace$, and assume $\norm{x} = n/2 \pm \alpha n$ for some $\alpha \in (0, 1/2]$. Let $\eps \in (0, \alpha)$.
   Then there exists a constants $M$ and $\tilde{d} > 0$ such that within $M$ steps, 
    with probability at least $\tilde{c} = 1 - e^{-\tilde{d}n}$, 
    the \oea\ will either have produced an individual $\tilde{x}$ such that $| \norm{\tilde{x}} - n/2 | < \eps n$
    or it will have produced an individual that is in a stepping stone.
\end{lemma}
\begin{proof}
  Let $\alpha_1 = \alpha$ and $x_1 = x$.
  Let $\rho$ be as in Lemma~\ref{lem:squishing_fn_f_using_rho}. By that lemma,
   there exists $\rho_1 \in (-1, 1)$ and function $f_1$, with $|\rho_1| \leq \rho$ and  $f_1(y) = \rho_1(y - 1/2) + 1/2$ such that
   with probability $c_1$ exponentially close to 1, the next individual $x_2$ (found by mutating $x_1$) satisfies
   \[
     \norm{x_2} = f_1(1/2 + \alpha_1) n,
   \]
   and so we have $\alpha_2 = \rho_1 \alpha_1$, or in other symbols, $|\norm{x_2} - n/2| = \alpha_2 n$.
   
   We will inductively repeat the previous paragraph. For step $k$, 
   if $x_k$ is in a stepping stone, we are done. Otherwise, let $|\norm{x_k} - n/2| = \alpha_k n$.  Because $|\rho_{k-1}| < 1$, we have
   that $\alpha_k < \alpha_{k-1}$. If $\alpha_{k} < \eps$, then we are done. So assume $\alpha_k \geq \eps$.
   The argument in the previous paragraph gives $\rho_k \in (-1, 1)$ and function $f_k$,
   with $|\rho_k| \leq \rho$ and $f_k(y) = \rho_k(y - 1/2) + 1/2$ such that with probability $c_k$ exponentially close to 1, the next
   individual $x_{k+1}$ satisfies
   \[
     \norm{x_{k+1}} = f_k(1/2 + \alpha_k) n,
   \]
   and so we have $\alpha_k = \rho_k \alpha_{k-1}$, or in other symbols, $| \norm{x_k} - n/2| = \alpha_k n $.
   
   Define $f^{(k)}(y) = (y - 1/2) \prod_{j=1}^k \rho_j + 1/2$. We have then that 
   \[
     \alpha_k + 1/2 = f^{(k)}(\alpha_1 + 1/2).
   \]
   By Lemma~\ref{lem:calling_fns_f_repeatedly_const_times}, there exists $M = M(\rho, \eps)$ such that for $t \geq M$,
  \[
    |f^{(t)}(\alpha_1 + 1/2) - 1/2| < \eps.
  \]
  We are done by noting that since $M$ is a constant, $\prod_{j=1}^{M} c_j$ is exponentially close to 1 (by, say, Bernoulli’s inequality, and replacing
  each $c_j$ with the smallest factor to get $(1 - e^{-dn})^M \geq 1 - Me^{-dn} > 1 - e^{\tilde{d}n}$ for some $\tilde{d} > 0$).
\end{proof}

The following lemma is weaker than it could be to make the proof easier (since a stronger lemma is not needed).
\begin{lemma}
\label{lem:hitting_time_of_any_stepping_stone}
  Let the mutation rate be $p$, and let $T_S$ denote the hitting time for $\cup_{k=1}^N S_k$, the time to reach any of the stepping stones.
  Assume that the EA always accepts every offspring and that no state is absorbing. 
  Then
  \[
     \expectation[T_S] = O\left(\frac{1}{(p^{p}(1-p)^{1-p})^n}\right).
  \]
\end{lemma}
\begin{proof}
  Let $\eps > 0$.
  Let $x$ be the current individual, even if $|\norm{x} - n/2| \geq \eps n$, by Lemma~\ref{lem:in_constant_steps_in_the_middle},
  for some integer $M$,
  with probability $\tilde{c}$ exponentially close to 1, an individual $x'$ will be produced within $M$ steps that is either in a stepping stone or has
  $|\norm{x'} - n/2| < \eps n$. Note that this is the case no matter what the current individual is.
  
  So, no matter what the current individual is, after at most $M$ steps, with probability $\tilde{c}$, 
 we have an individual $x$ such that $\norm{x} = n/2 + \eps' n$ for some $|\eps'| < \eps$.
  Let $p_{M}$ denote the probability that the next individual is in the biggest stepping stone $S_N$ (the set each of whose points have
  $\lfloor n/2 \rfloor$ ones). One way to get to $S_N$ from $x$ is by flipping exactly $(\eps' + p)n/2$ ones and $(p - \eps')n/2$ zeros,
   because 
   \[
     n/2 + \eps' n - (\eps' + p)n/2 + (p - \eps')n/2 = n/2.
   \]
   Then
   \[
     \begin{aligned}
      p_M &\geq \binom{(1/2 + \eps')n)}{(\eps' + p)n/2} \binom{(1/2 - \eps' n)}{(p - \eps')n/2} \left(p^p(1-p)^{1-p} \right)^n \\
              &\geq (p^p(1-p)^{1-p})^n.
     \end{aligned}
   \]
   
   By Lemma~\ref{lem:attempts_stochastic_domination} with $q = (p^p(1-p)^{1-p})^n$, we may conclude that
   \[
     \expectation[T_S] \leq \frac{M}{\tilde{c}(p^p(1-p)^{1-p})^n}.
   \]
   Since $M$ is constant and $\tilde{c}$ is bounded away from 0 (and in fact approaches 1 exponentially fast as $n \to \infty$), we
   are done.  
\end{proof}

\begin{lemma}
  \label{lem:attempts_stochastic_domination}
  Let $\{X_k\}_{k=1}^\infty$ be a sequence of Bernoulli trials with 
  conditional probabilities of success $\{p_k\}_{k=1}^\infty$ respectively (where $p_k$ is the probability that 
  $X_k = 1$ given $X_j = 0$ for all $j < k$). Suppose that the probabilities 
  $\{p_k\}_{k=1}^\infty$ are chosen randomly according to some discrete distribution (i.e., with only
  countably many possible values).
  Let $q \in (0, 1)$, and let $M \geq 1$ be an integer. For $k \geq 1$, let $A_k$ be the event defined as
  $A_k = \cup_{j=k+1}^{k+M}\{p_j \geq q\}$.
  Let $c \in (0, 1)$, and suppose for all $k \geq 1$, that for all $q_0 \in (0, 1)$ with $\Pr(p_{k-1} = q_0) > 0$, 
  we have 
  \[
      \Pr(A_k \mid p_{k} = q_0) \geq c.
  \]
 Let $T = \min \{k \mid X_k = 1 \}$ be the waiting time for the first success. Then
  \[
     \expectation[T] \leq \frac{M}{cq}.
   \]
 \end{lemma}
\begin{proof}
  Let $\{Y_k\}_{k=1}^\infty$ be a sequence of independent Bernoulli trials, each with probability of success $cq$.
  Let $T_Y = \min \{k \mid Y_k = 1\}$ be the waiting time for the first success. Consider the scaled random variable $MT_Y$.
  Notice that $MT_Y$ stochastically dominates $T$. (A definition of stochastic domination can be found in Definition 1.8.1 in
  \cite{DoerrN20}.) Therefore, %DoerrN20
  \[
    \expectation[T] \leq \expectation[MT_Y] = \frac{M}{cq}.
  \]
\end{proof}

% The following comment is no longer true:
%The following lemma is on the probability of jumping directly to the middle in one step, starting with $\delta n$
%ones more than $n/2$ for some small $\delta$.

\begin{lemma}
\label{lem:prob_to_get_to_next_step}
  This uses notation from Section~\ref{sec:notation_and_setup}.
  Let $x \in S_k$, where $k \in \{1, \ldots, N\}$. There exist $\delta_1, \delta_2 \in [0, 1]$ such that
  for the next individual $x'$ to be in $S_{k-1}$ it is sufficient to flip $p\norm{x} + \delta_1$ of the ones of $x$
  and $p(n-\norm{x}) + \delta_2$ of the zeros of $x$. The probability of this happening is $\Omega(1/\alpha^n)$, where
  $\alpha = \frac{1}{p^{p}(1-p)^{1-p}}$. Thus, the probability of $x'$ being in $S_k$ given $x \in S_{k-1}$ is
   also $\Omega(1/\alpha^n)$.
\end{lemma}
\begin{proof}
  Going from $S_k$ to $S_{k-1}$ means going from $\lfloor n s_k \rfloor$ ones to $\lfloor n s_{k+1} \rfloor$ ones, and recall
  $\lfloor n s_{k+1} \rfloor = \lfloor n \fp(s_k) \rfloor$. Let $\eps_1, \eps_2 \in (-1, 0]$ be such that
  \[
    \lfloor n s_k \rfloor = n s_k + \eps_1, \text{\quad and \quad} \lfloor n \fp(s_k) \rfloor = n \fp(s_k) + \eps_2.
  \]
  Therefore,
  \[
    \begin{aligned}
       \lfloor n s_k \rfloor - \lfloor n \fp(s_k) \rfloor &= n s_k + \eps_1 - (n \fp(s_k) + \eps_2) \\
       &= n s_k + \eps_1 - n(p(1-s_k) + (1-p)s_k) - \eps_2 \\
       &= n s_k - np + n p s_k - n s_k + np s_k + \eps_1 - \eps_2 \\
       &= np s_k - np + n p s_k +  \eps_1 - \eps_2  \\
       &= n p s_k - n p (1- s_k) +  \eps_1 - \eps_2 
    \end{aligned}
  \]
  Notice that $\lfloor n s_k \rfloor - \lfloor n \fp(s_k) \rfloor = \norm{x} - \norm{x'}$, which equals the number of ones flipped
  minus the number of zeros flipped.
  Ignoring integer rounding, we thus see that to go from $\lfloor n s_k \rfloor$ ones to $\lfloor n \fp(s_k) \rfloor$ ones, 
  it is sufficient to flip $np s_k$ ones of $x$ and $n p (1 - s_k)$ zeros of $x$. That means flipping a fraction of $p$
  of the ones of $x$ and $p$ of the zeros of $x$. Let $\delta_1$ and $\delta_2$ account for rounding errors.
  Flipping $p\norm{x} + \delta_1$ ones of $x$ and $p(n-\norm{x}) + \delta_2$ zeros means flipping $pn + \delta_1 + \delta_2$ bits
  and leaving the rest unflipped. This happens with probability 
  $\Omega(p^{pn + \delta_1 + \delta_2}(1-p)^{n(1 - p) - \delta_1 - \delta_2})$, and this equals
  $\Omega(1/\alpha^n)$, because $p^{\delta_1 + \delta_2}(1-p)^{-\delta_1-\delta_2}$ does not depend on $n$.
\end{proof}

\begin{lemma}
\label{lem:time_to_next_stepping_stone}
  For $k = 0, 1,\ldots, N-1$, let $T_k$ be the number of (further) steps to reach the set $\cup_{j=0}^k S_j$ given that the 
  current individual $x$ satisfies $x \in S_{k+1}$, and
  let $T_N$ be the hitting time of $\cup_{j=0}^N S_j$. 
  Let $p$ be the mutation rate, and let $\alpha = \frac{1}{p^{p}(1-p)^{1-p}}$.
  Then for all $k$, $\expectation[T_k] = O(\alpha^n)$.
\end{lemma}
\begin{proof}
  The case $k=N$ is Lemma~\ref{lem:hitting_time_of_any_stepping_stone}. Assume $k < N$ and 
  %We will ignore the fitness function.
  %By definition of $S_k$, to go from $S_k$ to $S_{k-1}$ means going from an individual with $\lfloor n s_k \rfloor$ to
  %$\lfloor n f_p(s_k) \rfloor$ ones.
  that $x \in S_{k+1}$. Notice that $T_k$ is a geometric random variable and that by 
  Lemma~\ref{lem:prob_to_get_to_next_step}, the probability of success in any step is $\Omega(1/\alpha^n)$.
  Therefore, $\expectation[T_k] = O(\alpha^n)$.
  
\end{proof}

\begin{proof}[Proof of Theorem~\ref{thm:exponential_runtimes}]
  First, assume a mutation rate of $p$. By Lemma~\ref{lem:hitting_time_of_any_stepping_stone}, the time to reach $\cup_{k=1}^N S_k$
  is $O(\alpha^n)$. By Lemma~\ref{lem:time_to_next_stepping_stone}, the time to reach any stepping stone of higher fitness is bounded by
  $O(\alpha^n)$. By Lemma~\ref{lem:N_is_Theta_log_n}, there are $\Theta(\log(n))$ stepping stones. Therefore, $\expectation[T(p)] = O(\alpha^n \log n)$.
  
  Next, assume a mutation rate of $q$ with $q \neq p$. Let $S$ be the event that the EA visits a stepping stone before \ones. 
By Lemma~\ref{lem:stepping_stone_first}, $\Pr(S) \geq 1/2$. (Actually $\Pr(S)$ is exponentially close to 1, since $|S_N| = \Theta(2^n/\sqrt{n})$, 
but we don't need that.)
Using the law of total expectation,
\[
\begin{aligned}
  \expectation[T(q)] &= \expectation[T(q) | S] \Pr(S) + \expectation[T(q) | \overline{S}] \Pr(\overline{S}) \\
  			      &\geq \expectation[T(q) | S] \Pr(S)\\
			      &\geq \frac{1}{2} \expectation[T(q) | S].
\end{aligned}
\]

Let $\pso$ denote the probability of reaching \ones\ in the next step, given that the current individual $x$ is one of the
  stepping stones (so with $x \neq \ones$), and let $\eta = 1/(\pso)^{1/n}$. By Lemma~\ref{lem:markov_chain_lemma_stochastic_domination},
  $\expectation[T(q) | S] \geq 1/\pso = \Omega(\eta^n)$. To complete the proof, we need only show that $\eta > \alpha$.

  \textbf{Case} $p > 1/2$. As a subcase, assume $q \geq 1/2$. By Lemma~\ref{lem:p_greater_than_half_what_is_pso}, 
  $\pso = O((q^p(1-q)^{1-p})^n)$. By Lemma~\ref{lem:maximize_simple_fns}, $\eta > \alpha$. 
  
  Next, assume $q < 1/2$ (with $p > 1/2$ still).
  Then Lemma~\ref{lem:p_greater_than_half_what_is_pso} implies (using its notation of $\beta$) that
  $\pso = O((q^{1-\beta}(1-q)^{\beta})^n)$.  By Lemma~\ref{lem:ineqs_with_beta}, $\eta > \alpha$.
  
  \textbf{Case} $p < 1/2$. As a subcase, assume $q > 1/2$. 
  By Lemma~\ref{lem:p_less_than_half_what_is_pso},
  $\pso = O(\gamma^n)$ with $\gamma = \sqrt{q(1-q)}$. By Lemma~\ref{lem:sqrt_ineq}, $\sqrt{q(1-q)} < q^q(1-q)^{1-q}$. By Lemma~\ref{lem:maximize_simple_fns},
  $q^q(1-q)^{1-q} \leq p^p(1-p)^{1-p}$. Thus $\eta > \alpha$.
  
    Finally, assume $q \leq 1/2$ (with $p < 1/2$ still). By Lemma~\ref{lem:p_less_than_half_what_is_pso}, 
  $\pso = O((q^p(1-q)^{1-p})^n)$. By Lemma~\ref{lem:maximize_simple_fns}, $\eta > \alpha$.
\end{proof}

\section*{Acknowledgments}
I would like to thank Benjamin Doerr for some helpful comments on this paper.

\printbibliography

\end{document}